\documentclass[letterpaper, 10 pt, conference]{ieeeconf}

\usepackage{amsmath}
\usepackage{amssymb}
\usepackage{mathtools}

\usepackage{balance}

\usepackage{tabularx}
\usepackage{graphicx}
\usepackage{bm}
\usepackage{color}
\usepackage{float}
\usepackage{units}
\usepackage{url}
\usepackage{hyperref}
\hypersetup{hidelinks}
\usepackage{amsmath}
\usepackage{makecell}
\usepackage{cite}
\usepackage{contour}
\usepackage[table,dvipsnames]{xcolor}
\usepackage{amsthm}
\usepackage{algorithm, algorithmicx, algpseudocode}
\usepackage{booktabs, multirow, multicol}
\usepackage{adjustbox}
\usepackage{empheq}

\definecolor{lightgrayrow}{RGB}{240,240,240}
\definecolor{lightpinkrow}{RGB}{255,230,235}

\newcommand{\newsec}[1]{\vspace{0.1cm} \noindent \textbf{#1} }

\makeatletter 
\def\endfigure{\end@float} 
\def\endtable{\end@float}
\makeatother

\usepackage{subcaption}
\usepackage[]{graphicx}
\usepackage{caption}

\newcommand{\csf}{CSF}

\newtheorem{theorem}{Theorem}

\newtheorem{lemma}{Lemma}
\newtheorem{definition}{Definition}

\newtheorem{proposition}{Proposition}
\usepackage{lipsum}
\usepackage{subcaption}

\IEEEoverridecommandlockouts

\begin{document}
\title{\csf{}: Contextual Safety Filtering for Motion Generators}
\author{Lizhi Yang$^1$, Yiling Hou$^1$, Yao Tang$^1$, Junheng Li$^1$, Daniel Weng$^2$, Blake Werner$^1$ and Aaron D. Ames$^1$
\thanks{$^1$California Institute of Technology, CA 91125, USA}\thanks{$^2$New York University, NY 10012, USA}
 \thanks{This work is supported by The Dow Chemical Company project \#227027AW and in part by Technology Innovation Institute.}
}
\maketitle
\thispagestyle{empty}
\pagestyle{empty}
\begin{abstract}
Text-conditioned motion generators produce trackable whole-body motion, but they have no notion of scene-dependent safety: the same action may target an object or a person.
Existing safeguards either inspect the prompt, require labeled motion data, or enforce geometric constraints; therefore, they do not directly account for how scene context changes a motion's meaning.
We introduce contextual safety filtering (\csf{}), a training-free filter that grounds natural-language safety rules in safe and unsafe reference trajectories produced by the generator.
For each active rule, safe and unsafe reference trajectories define an affine safety value that a safe reference tracking CBF-QP enforces.
Across four pretrained generators with different architectures, \csf{} activates the intended rules in all explicit and scene-triggered unsafe cases and reduces the danger-event rate by up to $90\%$, while preserving $88-100\%$ of benign motions.
We demonstrate the complete system on a real-world Unitree G1, where it successfully prevents unsafe motions in a variety of scenarios, including interactions with humans and objects.
\end{abstract}

\section{Introduction}\label{sec:introduction}

Pretrained text-conditioned motion generators can map natural-language descriptions to whole-body motions that are smooth, physically feasible, and increasingly trackable by humanoid controllers~\cite{rempe2026kimodo,jia2026echo,li2026motionhiflow,zhao2026ardy}.
However, physical feasibility alone does not determine whether a robot should execute a generated motion.
Kicking a ball and kicking a nearby person can require nearly the same body motion, even though the target in the scene changes whether the action is permitted.
A whole-body controller will track either reference, so the safety decision must account for the scene as well as the motion.

Existing safety mechanisms address only parts of this distinction.
Language guardrails can be bypassed by adversarial or misleading prompts~\cite{robey2025jailbreaking}, and a prompt-only screen cannot verify whether the generated motion is directed toward a benign object or a person.
Both a motion classifier and \csf{} require a safety specification. 
However, a separately trained classifier learns it from labeled clips and must be retrained or re-validated when the generator changes. 
\csf{} specifies unsafe behavior directly through natural-language rules.
Inference-time guidance imposes spatial or physical objectives~\cite{karunratanakul2023gmd,yuan2023physdiff} that specify where or how the body should move, rather than whether the requested act is permitted.
Classical robot safety filters enforce geometric sets such as collision-free regions and joint limits~\cite{ames2019control}, whereas the relevant harm here is semantic rather than geometric.

\begin{figure}[t]
  \centering
  \includegraphics[width=\columnwidth]{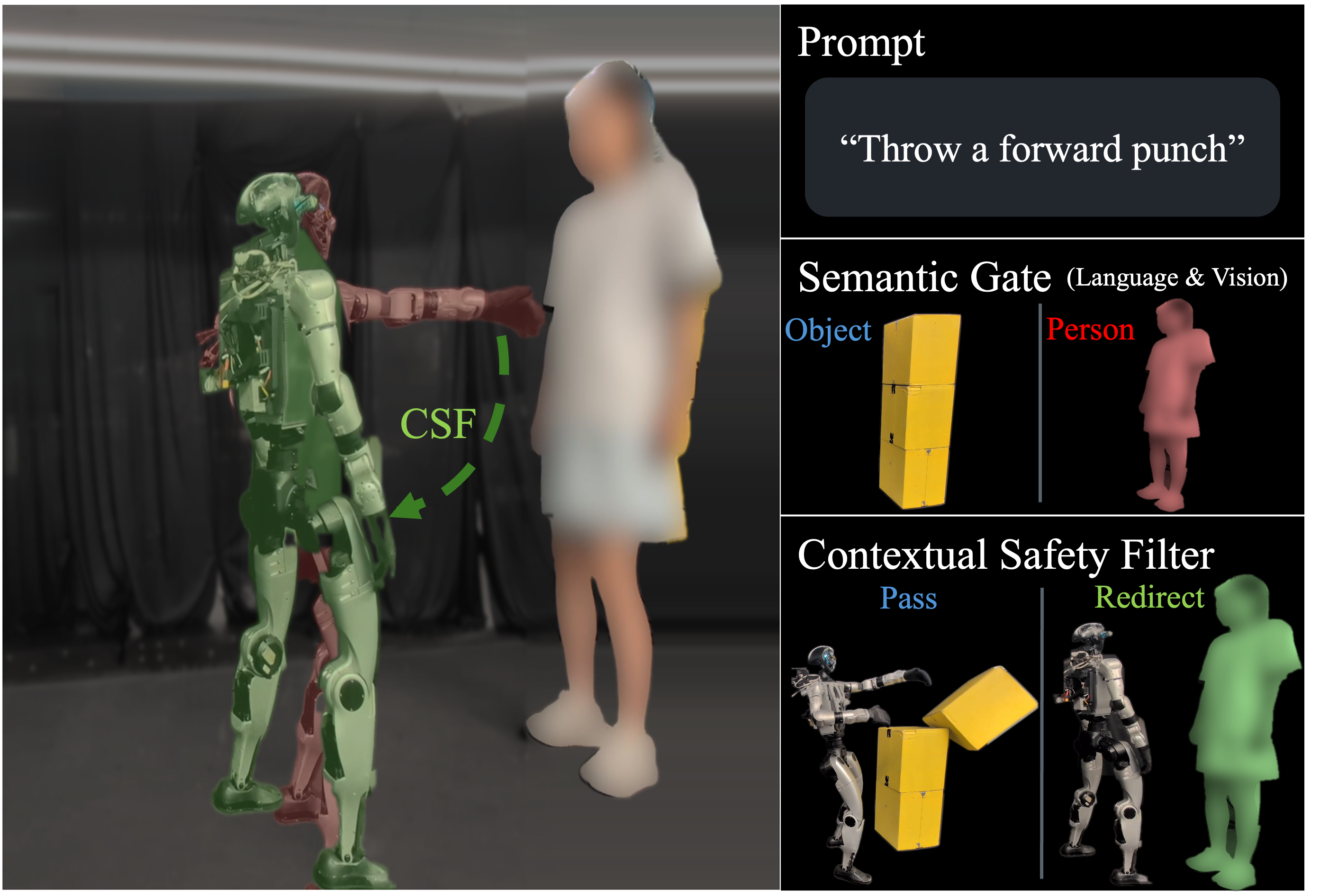}
  \caption{Hardware overview from real Unitree G1 frames.
  For the same prompt, ``Throw a forward punch,'' the Semantic Gate combines language and vision to distinguish an object from a person: \csf{} passes the punch toward the object and redirects it to a safe motion when a person is present.}
  \label{fig:head}
\vspace{-15pt}
\end{figure}

In this paper, we propose {\csf{}} (\textbf{C}ontextual \textbf{S}afety \textbf{F}iltering), a training-free semantic safety filter that uses outputs from the frozen generator as safety references.
For each rule, the generator produces a reference motion for the unsafe behavior and its safe alternative in the generator's native representation.
The difference between the references defines a semantic direction, and the signed displacement of the generator's current output estimate along this direction gives an affine safety margin.
A semantic context gate uses the command and scene entities to select which margins the QP enforces, allowing the same motion in a benign scene and filtering it when a protected entity is present (Fig.~\ref{fig:head}).
A compact CBF-QP enforces the active margins: its barrier-only solution is the minimum-norm correction, while the safe-reference tracking objective redirects the estimate toward the safe alternative.
The natural-language rule list provides the task-specific safety specification, and \csf{} trains neither an auxiliary safety model nor the motion generator.

\csf{} interfaces with each generator through the sampler's current estimate of its final output.
The rule selection, margin, and QP remain fixed, while a backbone adapter specifies where to read this estimate and how strongly to track the safe reference.
We evaluate \csf{} on Kimodo-G1~\cite{rempe2026kimodo}, ECHO~\cite{jia2026echo}, MotionHiFlow~\cite{li2026motionhiflow}, and ARDY~\cite{zhao2026ardy}.
We also deploy the ARDY-based system on a real-world Unitree G1 with a pretrained SONIC whole-body tracker~\cite{luo2025sonic}, where a runtime shield redirects the remaining reference when the scene changes.

Our contributions are as follows:
\begin{itemize}
\item We propose a training-free safety filter that uses natural-language rules, generator-produced reference motions, and scene context to constrain generated motion without training an auxiliary model.
\item We formulate a CBF-QP with a safe-reference tracking objective over an affine semantic margin, prove per-step feasibility, invariance, and optimality for any tracking strength, recover minimum intervention and exact pass-through at $\gamma=0$, and characterize the deployed ridge residual.
\item We evaluate the same filter across four pretrained generators and deploy the complete perception, generation, and tracking pipeline on a real-world Unitree G1, where the runtime shield redirects an executing action when a person enters view.
\end{itemize}

\section{Related Work}\label{sec:related}

\newsec{Text-conditioned motion generation.}
Text-conditioned motion generators span diffusion, flow matching, VAE-latent, and autoregressive models~\cite{guo2022generating,petrovich2022temos,tevet2023mdm,zhang2024motiondiffuse,chen2023mld,zhang2023t2mgpt,zhao2026ardy,li2026motionhiflow,rempe2026kimodo,cho2026safeflow}.
Whole-body policies can track the resulting motions on humanoids~\cite{cheng2024exbody,luo2025sonic}.
Yet a well-generated, trackable motion may still be inappropriate for the scene.
\csf{} addresses this gap by filtering the generator's current output estimate in the representation exposed by its sampler.

Scene-conditioned generation jointly represents language, motion, and a 3D environment to produce an action around a specified object~\cite{wang2022humanise}.
This context improves where and how an interaction is generated, but it does not decide whether the interaction is allowed.
\csf{} uses the command and perceived entities for that decision, while leaving motion generation to the pretrained model.
Motion-language models offer another source of semantics: \cite{tevet2022motionclip} aligns an autoencoder with CLIP space for generation, editing, and recognition, and \cite{petrovich2023tmr} learns a contrastive space for text-motion retrieval.
Such an encoder could support a learned motion classifier, but it remains a separately trained representation.
\csf{} instead obtains the unsafe and safe directions from each generator's own outputs and uses the rule list as the safety specification.

\newsec{Inference-time control and guidance.}
Pretrained motion generators can already be steered at test time.
Spatial guidance, physical projection, and composed diffusion priors move a sample toward a desired trajectory or feasible interaction~\cite{karunratanakul2023gmd,xie2024omnicontrol,yuan2023physdiff,shafir2024priormdm}, while noise and latent optimization offer similar control through motion-space objectives or programmed error functions~\cite{karunratanakul2024dno,liu2024progmogen}.
Another option is to generate several motions and select one using feasibility and text-motion checks~\cite{cao2026texedo}.
In each case, the objective specifies what motion to produce or which completed sample to keep.
Our question comes after that choice: whether the requested motion is appropriate for the scene in which it will be used.
\csf{} therefore lets the scene choose a semantic rule and intervenes only when the current output estimate violates it.

\newsec{Generative-model safety.}
Generative safety has largely focused on suppressing a concept either during sampling or in the model itself.
For image generation, safe guidance steers samples away from inappropriate content~\cite{schramowski2023sld}, while editing changes model components so that specified concepts, styles, or content are no longer produced~\cite{gandikota2023erasing,kumari2023ablating,gandikota2024uce}.
Motion unlearning adopts the same model-editing view for violent or otherwise unsafe actions~\cite{dematteis2026unlearning,wang2026safemo}.
This is effective when the concept itself is unwanted, but changing shared parameters can also alter benign generations that use the same representation.
Gating avoids changing the weights, although it still makes an all-or-nothing decision: \cite{cho2026safeflow} screens semantic out-of-distribution prompts, unstable trajectories, and hard kinematic violations before execution.
Adversarial text-to-motion attacks further show why the generated motion must be checked rather than inferred from the prompt alone~\cite{miao2025alert}.
\csf{} keeps the generator fixed and corrects only the scene-active violation, preserving the same action when no protected entity is at risk.

\newsec{Robot safety filters.}
Control barrier functions encode state-space safety through forward-invariant sets~\cite{ames2017cbf,ames2019control}, and safety filters enforce those sets by correcting a proposed action only when needed~\cite{hsu2024safetyfilter}.
Shields apply the same principle to formal action specifications~\cite{alshiekh2018shielding}.
Safety filtering can also shape a learned policy rather than remain in the deployment loop: \cite{yang2025cbfrl} filters RL rollouts during training so the policy internalizes CBF constraints, and \cite{yang2026pacman} combines training-time CBF guidance with onboard perception for whole-body humanoid avoidance.
Generative planners use the same principle to impose physical or task constraints during diffusion and flow sampling through barrier conditions, projection, or trajectory optimization~\cite{xiao2025safediffuser,yang2026safeflowmatcher,christopher2024projected,li2025hardflow}.
Language-level safeguards act earlier: \cite{ren2023knowno} asks for human help when a planner is uncertain, while \cite{yang2024safetychip} removes actions that violate temporal-logic constraints.
Together, these approaches learn policies around physical barriers, constrain generative plans, or screen language-level decisions.
\csf{} addresses the continuous motion that follows: it filters a frozen generator against natural-language rules and applies the same scene-selected margins to the remaining reference if the scene changes during execution.

\newcommand{\xhat}{\hat{x}}
\newcommand{\xsafe}{x_{\mathrm{safe}}}
\newcommand{\xunsafe}{x_{\mathrm{unsafe}}}

\section{Method}\label{sec:methods}

\subsection{Problem Setting}\label{sec:setting}

Consider a text-conditioned generator that samples a motion for prompt $c$ in a scene with semantically rich entities; its parameters remain frozen during filtering.
We declare unsafe behaviors in natural language and, at inference time, correct only motions that violate a scene-active rule while passing all other motions through without retraining or auxiliary models.
Our method relies on the following property of the sampler: it maintains a state $s_t$ that it refines toward a final motion and exposes its current estimate at every step:

\begin{definition}[Output estimate]\label{def:output-estimate}
At sampler state $s_t$ for prompt $c$, the \emph{output estimate} $\xhat(s_t,c)$ is the sampler's current estimate of the motion it will return when sampling terminates.
\end{definition}

The form of $\xhat$ depends on the sampler: it is the scheduler's clean estimate for diffusion~\cite{song2021ddim}, a velocity-based estimate for flow~\cite{lipman2023flow}, or the current uncommitted window for autoregression.
In each case, \csf{} replaces $\xhat$ with a filtered estimate $\xhat_{\mathrm{filt}}$ in the sampler's native update, leaving the prompt unchanged.
Figure~\ref{fig:overview} summarizes this interface at generation and execution time.

\begin{figure*}[t]
  \centering
  \includegraphics[width=\textwidth]{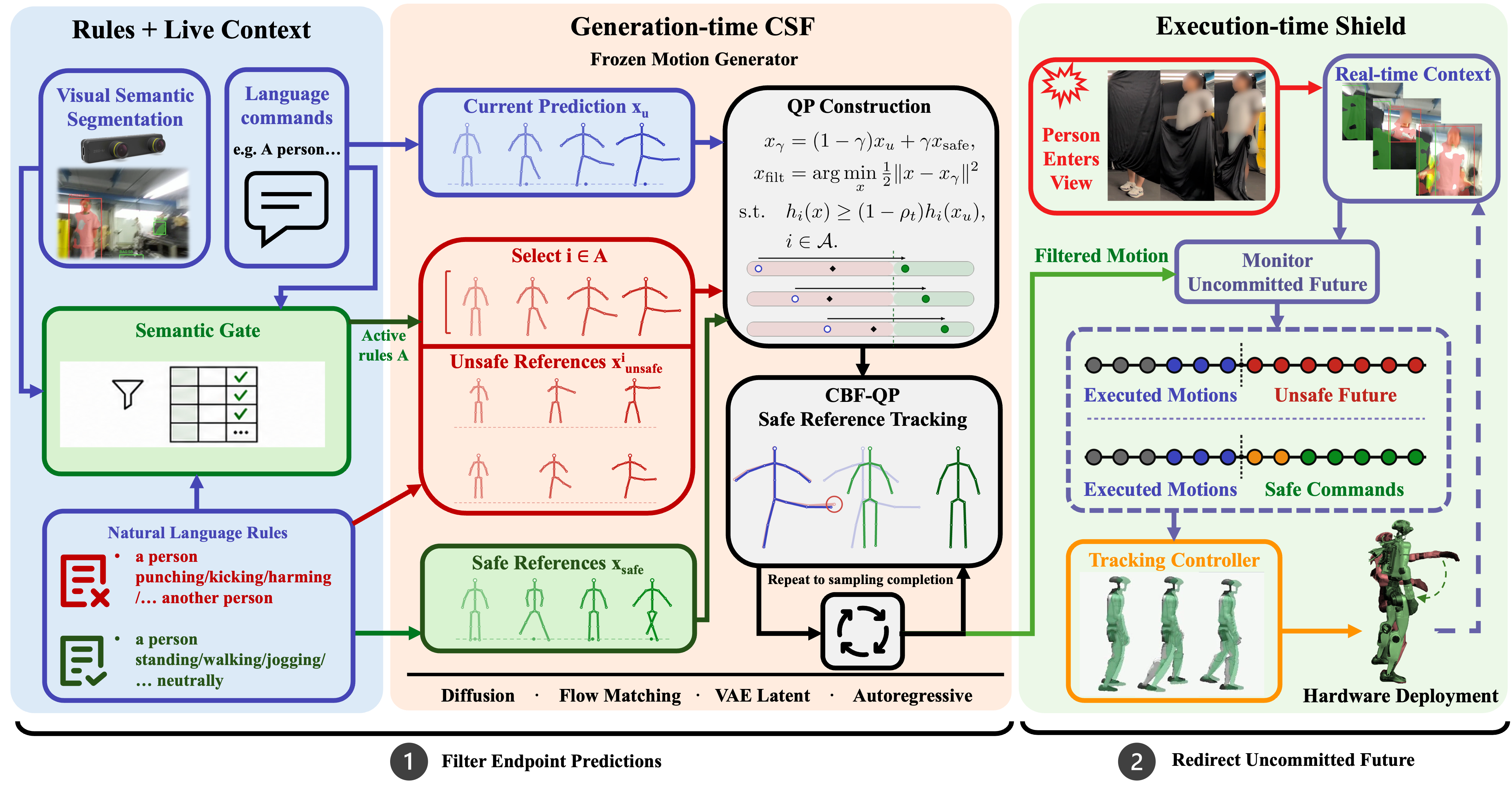}
  \caption{During generation, rules and scene context select the constraints enforced by a safe-reference tracking CBF-QP inside the frozen generator.
  During execution, a context change preserves the executed motion and redirects the remaining reference before a whole-body controller tracks it on Unitree G1.}
  \label{fig:overview}
  \vspace{-15pt}
\end{figure*}

\subsection{Semantic References and Margin}\label{sec:margin}
To construct the semantic margin, we define $K$ rules, each consisting of an unsafe behavior and the entity it protects.
For rule $i\in\{1,\ldots,K\}$, the frozen generator produces an unsafe reference $\xunsafe^{(i)}$ from the rule description and a safe reference $\xsafe$ from a benign version of the current command that preserves the rest of the requested motion.
\csf{} solves the QP for every current output estimate; the motion-space diffusion path forms step-matched reference estimates at the current sampler state, while the remaining paths generate references once in their prediction spaces and reuse them at each step (Sec.~\ref{sec:decoder}).
Because the references and prediction share the same representation, the margin requires neither embedding alignment nor paired data.
The difference $\xunsafe^{(i)}-\xsafe$ defines the semantic direction, while the protected entity determines when the rule is active.
Changing a rule only regenerates its references without updating the generator.
We therefore define the semantic margin as the negative projection of the current estimate's displacement from the safe reference onto the unsafe direction.
\begin{definition}[Semantic margin]\label{def:margin}
For a well-defined rule $i$ with $\|\xunsafe^{(i)}-\xsafe\|>0$, the semantic margin of the current output estimate $\xhat$ is
\begin{equation}\label{eq:margin}
 m_i(\xhat)=-\frac{\left\langle \xhat-\xsafe,\;
 \xunsafe^{(i)}-\xsafe\right\rangle}
 {\left\|\xunsafe^{(i)}-\xsafe\right\|}.
\end{equation}
\end{definition}

A lower margin places the current estimate farther along the unsafe direction for concept $i$.
Equation~\eqref{eq:margin} has three direct properties.

\begin{lemma}[Margin affinity]\label{lem:affine}
Let $a_i=-(\xunsafe^{(i)}-\xsafe)/\|\xunsafe^{(i)}-\xsafe\|$.
Then: (i) $m_i(x)=a_i^{\top}(x-\xsafe)$; (ii) $m_i(\xsafe)=0$; and (iii) a common positive rescaling changes the margin magnitude but not its sign.
\end{lemma}

This makes the margin useful as a barrier: each rule defines a half-space, and moving an estimate a fraction $\alpha$ toward $\xsafe$ scales its margin by $1-\alpha$.
Because all rules use the same safe reference, this interpolation provides a common feasible direction when several rules are active.
The choice of $\xsafe$ also determines where a correction leads.
Because an unconditional mean would pull the motion toward a static pose, a text encoder selects the benign command from a library of neutral actions, with a neutral stand as the fallback.

\subsection{Safe-Reference Tracking CBF-QP}\label{sec:qp}

We now use the affine margins as barrier constraints and the safe reference to guide the correction of the sampler's current estimate.
We call a concept \emph{scene-active} when the scene or command implicates the entity it protects, and let $\mathcal{A}$ denote the set of active concepts.
Thus we define the semantic barrier function and safe set as follows.
\begin{definition}[Semantic barrier function]\label{def:barrier}
The barrier function of concept $i$ is the margin itself,
\begin{equation}\label{eq:barrier}
 h_i(x)\coloneqq m_i(x),
\end{equation}
and the safe set is the intersection of the active concepts' superlevel sets, $\mathcal S=\{x:h_i(x)\ge0,\ i\in\mathcal A\}$.
\end{definition}

We therefore formulate the safety filter as a discrete-time CBF-QP with a safe-reference tracking objective~\cite{ames2017cbf,agrawal2017dcbf}:
\begin{equation}\label{eq:qp}
\begin{aligned}
 \xhat_{\mathrm{filt}}=\arg\min_{x}\quad &\tfrac{1}{2}\big\|x-\big((1-\gamma)\xhat+\gamma\xsafe\big)\big\|^2\\
 \mathrm{s.t.}\quad &h_i(x)\ge(1-\rho_t)\,h_i(\xhat),\ \ i\in\mathcal{A},
\end{aligned}
\end{equation}
Here $\rho_t\in[0,1]$ is the barrier contraction rate and $\gamma\in[0,1]$ weights the safe-reference tracking objective.
A nonzero tracking weight trades exact pass-through for redirection toward the safe reference.
The sampler then uses $\xhat_{\mathrm{filt}}$ in place of $\xhat$ for its native update.
Because the references remain fixed during the solve, each $h_i$ is affine in $x$ by Lemma~\ref{lem:affine}(i), and Eq.~\eqref{eq:qp} is a convex QP.
At $\rho_t=1$, the constraint imposes the declared safe set; a smaller rate contracts the current margin gap.

\begin{theorem}[Per-step guarantees]\label{prop:feas}
Let $\rho_t>0$, $\gamma\in[0,1]$, $y_\gamma=(1-\gamma)\xhat+\gamma\xsafe$, and let $\mathcal C_t(\xhat)$ be the feasible set of Eq.~\eqref{eq:qp}.
Let $\mathcal Q_{\mathcal A}(\xhat)=\xhat_{\mathrm{filt}}$ denote its exact solution.
(i) $\mathcal C_t(\xhat)$ is nonempty because $(1-\rho_t)\xhat+\rho_t\xsafe\in\mathcal C_t(\xhat)$.
(ii) The filtering map renders $\mathcal S$ invariant, $\mathcal Q_{\mathcal A}(\mathcal S)\subseteq\mathcal S$, and contracts each violation according to $h_i(\xhat_{\mathrm{filt}})\ge(1-\rho_t)h_i(\xhat)$, so $\xhat_{\mathrm{filt}}\in\mathcal S$ in one step when $\rho_t=1$.
(iii) $h_i(y_\gamma)=(1-\gamma)h_i(\xhat)$, and $\xhat_{\mathrm{filt}}$ is the Euclidean projection of $y_\gamma$ onto $\mathcal C_t(\xhat)$.
At $\gamma=0$, this is the minimum-norm correction of $\xhat$ and leaves feasible proposals unchanged.
\end{theorem}
\begin{proof}
For (i), Lemma~\ref{lem:affine}(ii) gives $h_i(\xsafe)=0$.
Affinity then gives $h_i\big((1-\rho_t)\xhat+\rho_t\xsafe\big)=(1-\rho_t)h_i(\xhat)$, so every constraint holds with equality.
For (ii), the constraint bound $(1-\rho_t)h_i(\xhat)$ is nonnegative when $h_i(\xhat)\ge0$.
Otherwise, it contracts the violation by $1-\rho_t$.
For (iii), affinity and $h_i(\xsafe)=0$ give
$h_i(y_\gamma)=(1-\gamma)h_i(\xhat)+\gamma h_i(\xsafe)=(1-\gamma)h_i(\xhat)$.
Equation~\eqref{eq:qp} projects $y_\gamma$ onto the nonempty closed polyhedron $\mathcal C_t(\xhat)$.
The projection is unique, and at $\gamma=0$ it fixes exactly the compliant proposals~\cite{boyd2004convex}.
\end{proof}
\begin{figure}[t]
\centering
\includegraphics[width=\columnwidth]{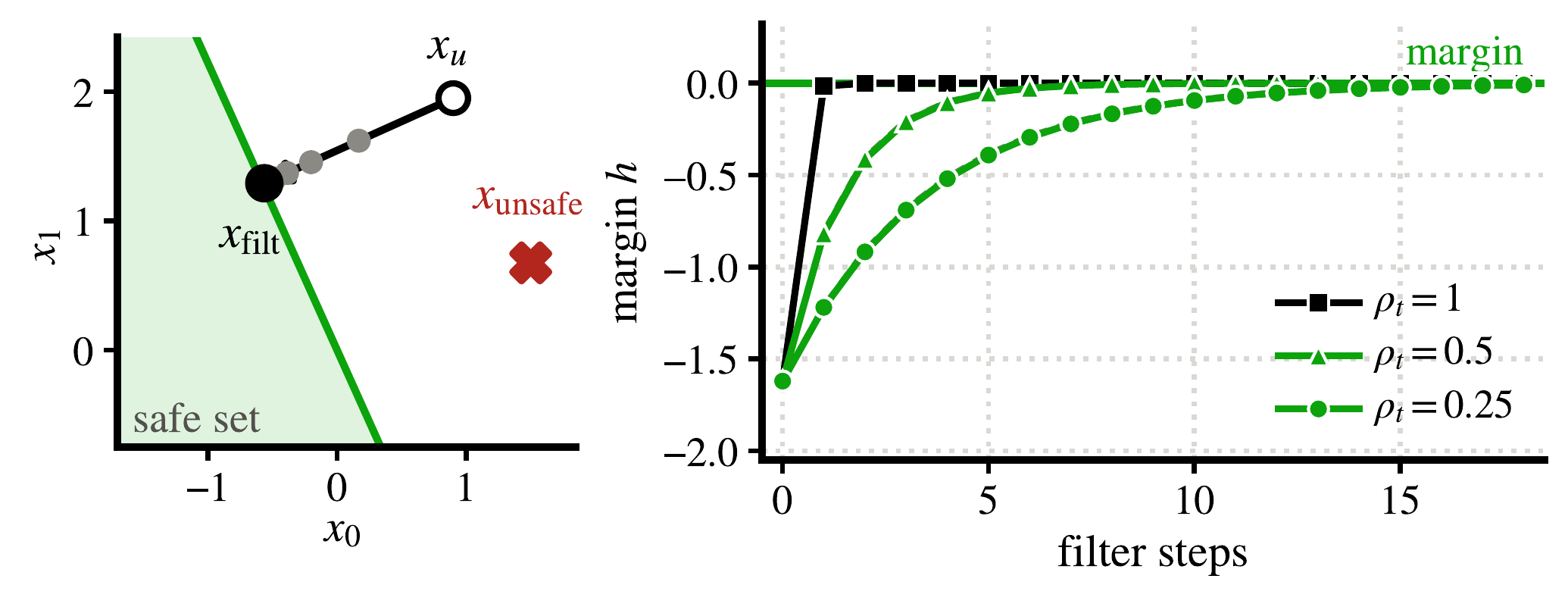}
\caption{A two-dimensional example of the barrier-only QP step.
A violating estimate takes the shortest route onto the margin at $\rho_t=1$, while a lower contraction rate enforces the fractional gap contraction shown in dashed lines.}
\label{fig:steps}
\vspace{-15pt}
\end{figure}
Figure~\ref{fig:steps} illustrates the barrier-only projection in two dimensions.

Let $c_s=\xsafe-\xhat$ and $r_i=-\rho_t h_i(\xhat)$.
Equation~\eqref{eq:qp} becomes $\min_{\xhat_{\mathrm{filt}}}\tfrac12\|\xhat_{\mathrm{filt}}-(\xhat+\gamma c_s)\|^2$ subject to $A(\xhat_{\mathrm{filt}}-\xhat)\ge r$, where the rows of $A$ stack the active directions $a_i^{\top}$.
The primal has one variable per channel of the window, so its size follows the representation used by the generator, whereas the number of constraints equals the number of active concepts.
We therefore solve the dual, where the iteration remains in the space of declared rules and $\xhat_{\mathrm{filt}}=\xhat+\gamma c_s+A^{\top}\lambda$.
At $\gamma=0$, the change $\xhat_{\mathrm{filt}}-\xhat$ lies in the span of the active semantic directions.
Because active concept directions can be highly correlated, we stabilize the small dual with a ridge relaxation, which has an exact hinge-penalty characterization.

\begin{proposition}[Ridge relaxation and its residual]\label{prop:ridge}
For $\varepsilon>0$, maximizing the ridge dual $-\tfrac12\lambda^{\top}(AA^{\top}+\varepsilon I)\lambda+\lambda^{\top}(r-\gamma Ac_s)$ over $\lambda\ge0$ and setting $\xhat_{\mathrm{filt}}=\xhat+\gamma c_s+A^{\top}\lambda$ solves the hinge-penalized program $\min_{x}\tfrac12\|x-(\xhat+\gamma c_s)\|^2+\tfrac{1}{2\varepsilon}\sum_{i\in\mathcal{A}}[(r_i-a_i^{\top}(x-\xhat))_{+}]^2$, and at the optimum each constraint with $\lambda_i^{\star}>0$ is violated by exactly $\varepsilon\lambda_i^{\star}$.
\end{proposition}
\begin{proof}[Proof sketch]
The penalized program is convex with the ridge expression as its dual, which is strongly concave, so cyclic coordinate ascent~\cite{hildreth1957quadratic} converges to the unique maximizer; stationarity at $\lambda_i^{\star}>0$ gives $r_i-a_i^{\top}(\xhat_{\mathrm{filt}}-\xhat)=\varepsilon\lambda_i^{\star}$.
\end{proof}

The dual solve has $|\mathcal A|$ variables; the representation dimension enters only through the batched inner products that form the Gram matrix and reconstruct $\xhat_{\mathrm{filt}}$.
The ridge introduces a controlled residual.
At $\rho_t=1$, a single violated unit row leaves
\begin{equation}\label{eq:ridge-single}
 h_i(\xhat_{\mathrm{filt}})=(1-\gamma)\frac{\varepsilon}{1+\varepsilon}h_i(\xhat).
\end{equation}
Thus, for the barrier-only case $\gamma=0$, the contraction ratio is $\varepsilon/(1+\varepsilon)$; with multiple rows, Proposition~\ref{prop:ridge} gives the exact residual $\varepsilon\lambda_i^{\star}$.
The cross-backbone benchmark in Sec.~\ref{sec:bench} shows that this relaxed solve retains the intended safety--preservation tradeoff in practice.

\subsection{Semantic Context Gate}\label{sec:gate}

The margin identifies the behavior in a motion, while the context gate uses the command and detected scene entities to select which protected-entity rules the QP should evaluate.
The \emph{prompt} signal activates the context-sensitive rule set when the command is closest to an unsafe description rather than a benign contrast.
The \emph{perception} signal is rule specific and activates rule $i$ when the label of a visually detected entity is closest to the entity protected by that rule.
The generator's text encoder performs these comparisons on the command and detected entity labels, so matching does not require a fixed vocabulary.
For example, ``punch the wall'' may appear benign from the command alone, but detecting a person in the scene selects the relevant rule.
Concept $i$ is active when either signal fires,
\begin{equation}\label{eq:gate}
 g_i=g_i^{\mathrm{perception}}\lor g_i^{\mathrm{prompt}},
 \qquad \mathcal{A}=\{i:g_i=1\},
\end{equation}
which supplies the constraints to Eq.~\eqref{eq:qp}; the QP changes the estimate only if it violates one of them, while an empty $\mathcal A$ gives $\xhat_{\mathrm{filt}}=\xhat$.

\subsection{Application Across Generators}\label{sec:decoder}

\newsec{Direct motion-space predictions.}
When the output estimate is in motion-feature space, we apply Eq.~\eqref{eq:qp} before each native sampler update.
Let $M$ be the diagonal projector onto the filtered channels and let $\mathcal Q_{\mathcal A}^{M}(\xhat)$ be the solution of Eq.~\eqref{eq:qp} evaluated with $M\xhat$, $M\xsafe$, and $\{M\xunsafe^{(i)}\}_{i\in\mathcal A}$.
The full filtered estimate is
\begin{equation}\label{eq:direct-filter}
 \mathcal F_{\mathcal A}(\xhat)
 =M\mathcal Q_{\mathcal A}^{M}(\xhat)+(I-M)\xhat.
\end{equation}
We exclude root channels so the filter must correct the unsafe pose content rather than satisfy a rule by translating the character away from the protected entity.
With classifier-free guidance, each path expresses the QP objective, constraint base point, and update in the same prediction space.

\newsec{Decoder-aware latent predictions.}
A generator may instead predict a compressed VAE latent $z$, where distances need not separate the declared concepts (Sec.~\ref{sec:bench}); we therefore evaluate the \emph{same} margin after the pretrained decoder.
With $\phi(z)=PD(z)$ denoting the decoded pose channels, Eq.~\eqref{eq:margin} becomes
\begin{equation}\label{eq:decmargin}
 \widetilde m_i(z)=-\left\langle\phi(z)-\phi(z_s),
 \frac{\phi(z_i)-\phi(z_s)}{\|\phi(z_i)-\phi(z_s)\|}\right\rangle.
\end{equation}
Because $D$ is nonlinear, we re-decode after every update.
Let $k=\arg\min_{i\in\mathcal A}\widetilde m_i(z)$ be the worst active margin and $q=\nabla_z\widetilde m_k(z)$ its gradient through the pretrained decoder; while $\widetilde m_k(z)<0$, we iterate
\begin{equation}\label{eq:decstep}
 z\;\leftarrow\;z-\eta\,\frac{\widetilde m_k(z)}{\|q\|^2+\varepsilon}\,q
\end{equation}
until the barrier is met or an iteration budget $N$ is exhausted.
On a pyramidal sampler, this decoder-aware path acts at the full-resolution stage where the references were generated.
We implement this decoded-space option.
For the latent-flow benchmark, however, we apply Eq.~\eqref{eq:qp} directly in the latent with $\gamma=0.5$ (Sec.~\ref{sec:bench}).

\newsec{Autoregressive window predictions.}
Let $H_j$ be the committed history, let $\xhat_j$ be the current estimate of uncommitted window $j$, and let $\mathcal F_{\mathcal A,j}$ denote Eq.~\eqref{eq:direct-filter} evaluated with the corresponding reference blocks.
Filtering and commitment follow
\begin{equation}\label{eq:ar-filter}
 \xhat_{j,\mathrm{filt}}=\mathcal F_{\mathcal A,j}(\xhat_j),
 \qquad H_{j+1}=H_j\oplus q(\xhat_{j,\mathrm{filt}}),
\end{equation}
where $q$ is the native token requantization and $\oplus$ appends the resulting window to the history that conditions the next window.
Because $\|a_i\|=1$, requantization changes each margin by at most $\|q(\xhat_{j,\mathrm{filt}})-\xhat_{j,\mathrm{filt}}\|$.
Because committed windows are immutable, any intervention propagates through this history; we therefore filter only pose content and leave the root trajectory to the generator.

\subsection{Runtime Safety Shield}\label{sec:shield}

Context can change while the generated motion is being tracked, so the runtime shield applies the same rules to the remaining reference.
Let $R$ be the tracking reference, $S$ a safe continuation, $\ell$ the current execution cursor, and $W_j$ a future window beginning at frame $s_j$.
For the active rule set $\mathcal A_\ell$, let $e=\min(\ell+b,T-1)$ be the splice frame after a short response buffer $b$.
The shield selects the first violating unexecuted window and updates the reference by
\begin{equation}\label{eq:runtime-shield}
\begin{aligned}
 j^\star&=\min\!\left\{j:s_j\ge\ell,
 \min_{i\in\mathcal A_\ell}h_i(R[W_j])<\tau\right\},\\
 R'_t&=\begin{cases}
 R_t, & t<e,\\
 \mathcal B_e(R,S)_t, & t\ge e,
 \end{cases}
\end{aligned}
\end{equation}
Here $\tau$ is the calibrated runtime threshold, and $\mathcal B_e$ aligns $S$ to the held state in heading and position and blends it with $R$ over a fixed transition.
If no violating window exists, the shield leaves $R$ unchanged.
Joint velocities are recomputed over the edited frames, and the tracker continues on $R'$.

\section{Experiments}\label{sec:experiment}

We apply \csf{} without retraining to four pretrained generators through the output-estimate interface of Definition~\ref{def:output-estimate}: Kimodo-G1~\cite{rempe2026kimodo} (DDIM diffusion), ECHO~\cite{jia2026echo} (diffusion with DPM-Solver sampling~\cite{lu2022dpm}), MotionHiFlow~\cite{li2026motionhiflow} (flow matching in a VAE latent), and ARDY~\cite{zhao2026ardy} (windowed autoregressive generation over hybrid tokens).
Table~\ref{tab:backbones} lists each adapter's output space, tracking strength $\gamma$, and root-channel handling.
All adapters share the rules, margin, gate, and QP; when root channels are separable, we filter only the motion content.

\begin{table}[t]
\centering
\caption{Per-backbone settings: the space of each adapter's output estimate, its tracking strength $\gamma$, and whether the correction excludes explicit root channels.}
\label{tab:backbones}
\footnotesize
\setlength{\tabcolsep}{5pt}
\begin{tabular}{@{}llcl@{}}
\toprule
Backbone & Prediction space & $\gamma$ & Root channels \\
\midrule
Kimodo-G1      & $417$-d motion feature      & $0$   & excluded ($5$)  \\
ECHO           & $38$-d motion feature       & $0.2$ & not separable \\
MotionHiFlow   & VAE latent                  & $0.5$ & not separable \\
ARDY           & hybrid token ($20\!+\!128$) & $0.2$ & excluded ($20$) \\
\bottomrule
\end{tabular}
\vspace{-15pt}
\end{table}

The shared rule set contains six concepts: four for person-directed violence and one each for a child and an animal.
Examples include aggressively striking another person, violently striking a child, and kicking or stomping an animal, with benign contrasts such as striking training equipment or caring for a child.
Each backbone generates its own references, so exemplars and margin scales remain model specific and margins are not compared across backbones.

\begin{figure*}[t]
\centering
\includegraphics[width=\textwidth]{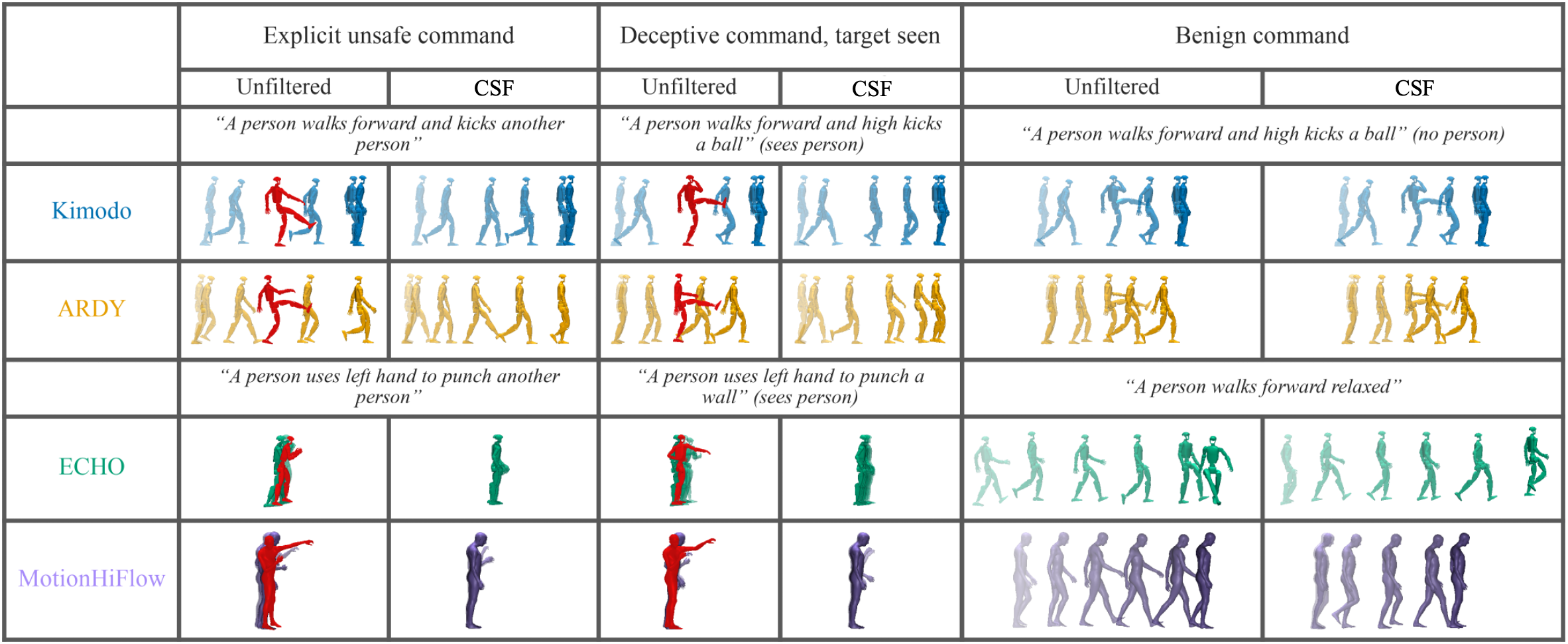}
\caption{Four seconds of motion generated with each generator, with the unfiltered output beside \csf{}; unsafe frames are colored red.
Each command row gives the prompt used for the two backbones below it.
The explicit case names a protected target.
The deceptive case names a benign target while a person is present, whereas the benign case contains no protected targets.}
\label{fig:strobe}
\end{figure*}

\begin{figure*}[t]
\centering
\includegraphics[width=\textwidth]{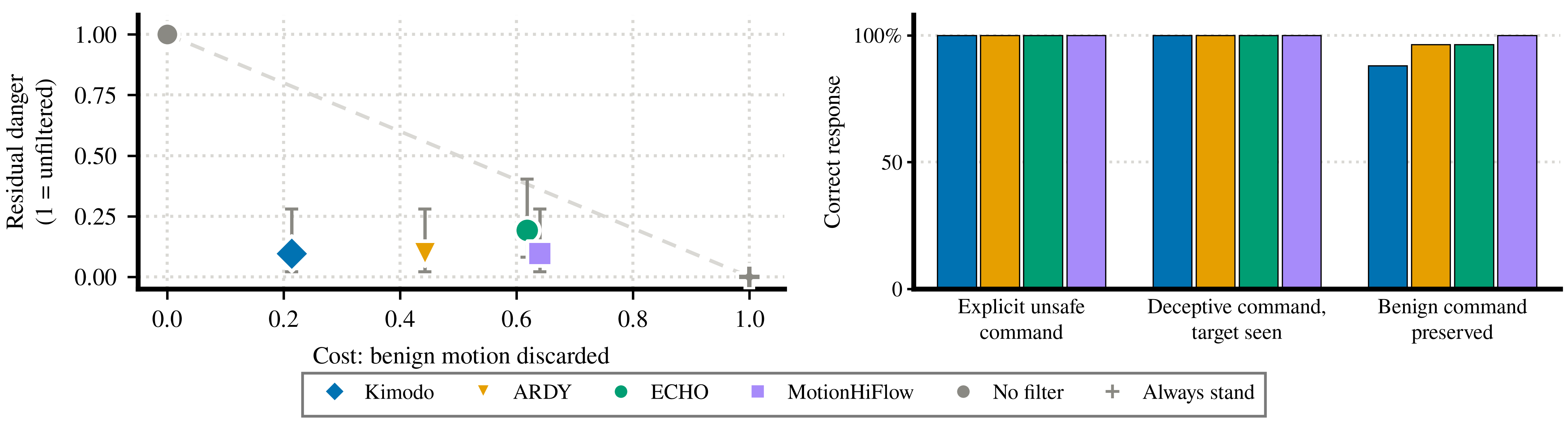}
\caption{Benchmark results.
Left: all four backbones fall below the line joining the degenerate filters, indicating greater danger suppression per unit of discarded motion than a blend toward a standing pose.
Right: correct response rate for each case, where higher is better; the first two cases test activation of the semantic gate from the command and declared scene entities, while benign motion counts as preserved when the correction stays within the generator's own draw noise or moves the clip closer to its selected safe counterpart.}
\vspace{-15pt}
\label{fig:results}
\end{figure*}

\begin{figure*}[t]
\centering
\includegraphics[width=\textwidth]{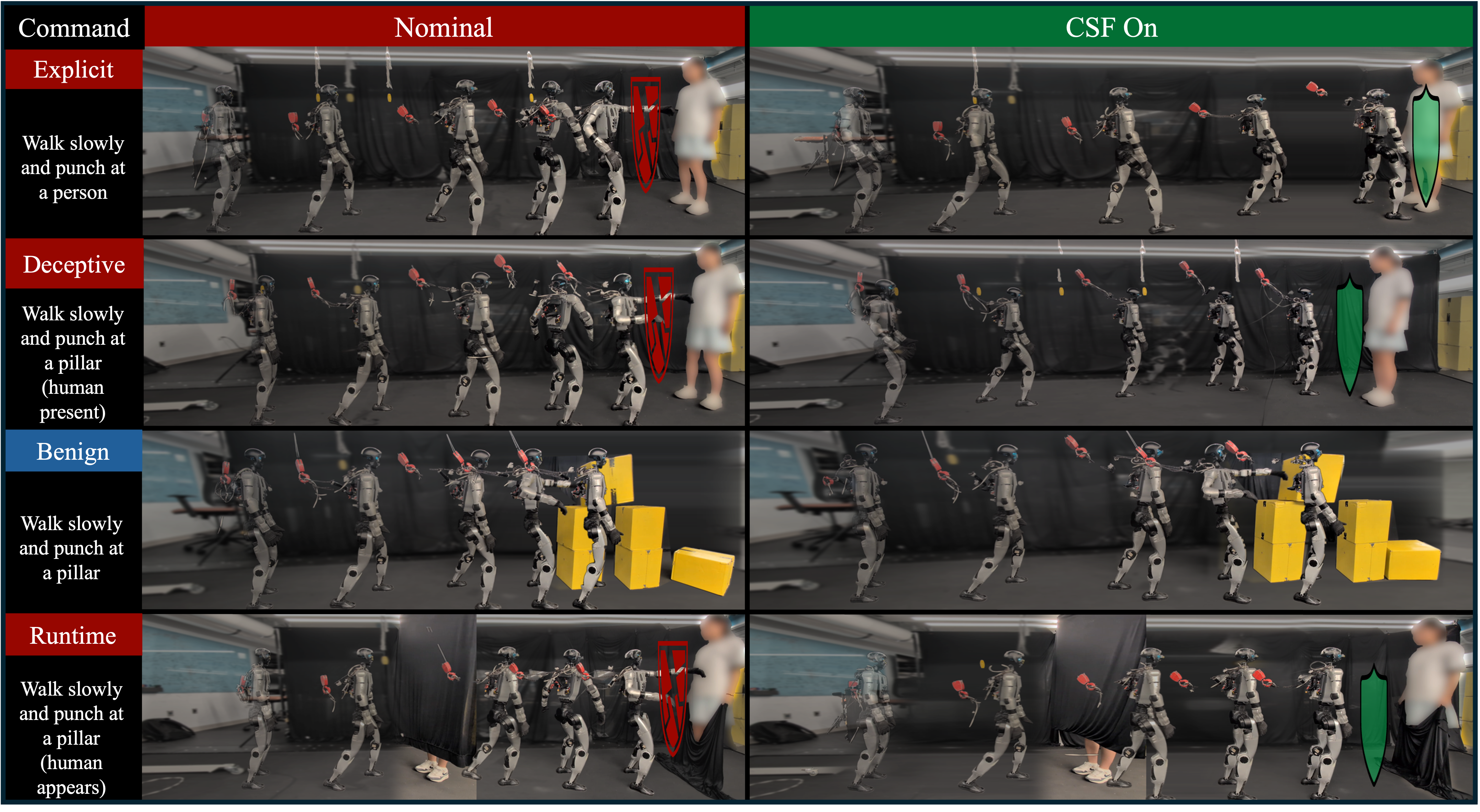}
\caption{Real-frame Unitree G1 trials with nominal execution (left) and \csf{} enabled (right).
Rows show an explicit person-directed command, a deceptive pillar-directed command with a person present, a benign pillar-directed command, and a person entering view during execution.
\csf{} blocks the first two punches, preserves the benign punch, and redirects the remaining reference in the runtime trial.}
\label{fig:hardware}
\vspace{-15pt}
\end{figure*}

\subsection{Benchmark}\label{sec:bench}

We construct 28 unsafe--benign command pairs over the six concepts and use three random seeds for each pair in allowed and protected scenes, with and without \csf{}.
The paired runs share their command and seed and differ only in declared scene entities, yielding 672 clips per backbone and 2,688 total: 1,344 filtered and 1,344 unfiltered.
Benign commands remove only the regulated sub-action, such as replacing a walking strike with a neutral walk, while scene and filtering pairs hold the command and seed fixed to separate rule activation from correction performance.

We detect residual danger independently of the semantic margin as persistent bursts of fast, extended effector motion above a threshold fixed at the unfiltered-benign 95th percentile before scoring the filtered outputs.
We also use Mahalanobis distance to quantify how far filtering shifts a benign motion toward standing.
Because margin units differ across generators, each backbone is normalized to its own no-filter and unconditional-standing endpoints, $(0,1)$ and $(1,0)$.
The reported danger ratio divides the filtered event rate by its unfiltered counterpart, so $0.1$ represents a $90\%$ relative reduction.

Residual danger falls to $0.1$ of the unfiltered rate on three backbones and $0.2$ on ECHO (Fig.~\ref{fig:results}, left).
Preservation cost ranges from $0.21$ on Kimodo to $0.64$ on MotionHiFlow and follows concept separation in the prediction space.
Kimodo's distinct motion-space margin supports minimum intervention, while MotionHiFlow's nearly flat latent margin needs a stronger safe-reference pull.
ECHO and ARDY lie between these cases.
Because the prompt and scene select the same QP rows for every backbone, these cost differences reflect the prediction representation and $\gamma$.
All backbones fall below the line joining the degenerate filters in Fig.~\ref{fig:results} (left), so each suppresses more danger per unit of discarded benign motion than a standing-pose blend.

The gate selects the intended rule for every explicit and scene-triggered unsafe case across all four backbones and leaves it inactive in the matched allowed scene (Fig.~\ref{fig:results}, right).
For benign motion, $88$ to $100\%$ of clips either pass through or reach the selected safe counterpart.
On Kimodo, the correction is smaller than the natural variation between two unfiltered samples generated from the same command.
Figs.~\ref{fig:strobe} and~\ref{fig:results} (right) show these outcomes.

\subsection{Inference-Time Cost}\label{sec:runtime-cost}
On one RTX~4090, the rule-space dual takes $0.01$ to $0.18$~ms for one to six active rules, and a complete filter call takes $0.3$ to $1.2$~ms.
ECHO, MotionHiFlow, and ARDY cache their references after a one-time setup of $0.6$ to $0.9$~s, adding no measurable per-sample overhead in the fastest case and at most $19\%$.
Kimodo constructs step-matched references during sampling, increasing the time needed to generate a four-second motion from $1.61$ to $7.60$~s, although the QP itself accounts for only $0.14$~s.
Reference construction, rather than the rule-space solve, therefore determines most of the architecture-dependent cost.

\subsection{Real-World Unitree G1 Deployment}\label{sec:exp-hardware}
We deploy the ARDY-based pipeline on a real-world Unitree G1.
A verbal command and camera perception supply the motion prompt and scene entities.
ARDY generates the reference, \csf{} filters it, and a pretrained SONIC policy~\cite{luo2025sonic} tracks the result.
The generator and tracker remain unchanged between nominal and filtered trials.

We test punch and kick commands in three conditions.
In the explicit condition, a person-directed command selects the rule from language, and \csf{} redirects the strike.
In the deceptive condition, an object-directed command appears benign, but a detected person selects the rule and \csf{} redirects the motion.
In the benign condition, only the object is present, so the rule remains inactive, and the requested motion executes.
The top rows of Fig.~\ref{fig:hardware} show that \csf{} blocks the explicit and scene-triggered person-directed punches while preserving the same skill in the benign scene.
\vspace{-1.5pt}
\subsection{Runtime Safety Shield on Hardware}\label{sec:exp-shield}
We also test runtime safety by introducing a context change after generation.
The robot begins an object-directed punch or kick sequence, then a person enters view during execution.
The updated entity list activates the person-protection rule and triggers the shield on the remaining reference.
It preserves the executed portion and introduces a safe continuation for the unchanged tracker.
The robot begins the allowed task but does not complete the strike after the person appears, as shown in Fig.~\ref{fig:hardware}.
Unlike generation-time filtering, the shield responds to a later scene change by redirecting the remaining reference in real time.

\section{Conclusion}\label{sec:conclusion}

\csf{} turns generator-produced unsafe and safe references into scene-selected affine margins enforced by a safe-reference tracking CBF-QP.
At $\gamma=0$, the QP makes the smallest correction needed to satisfy the active rules; increasing $\gamma$ redirects the estimate toward the safe alternative.
Across four pretrained generators, residual danger falls by up to $90\%$ with better preservation than a standing-pose blend; correction cost follows concept separation in each prediction space.
On a real-world Unitree G1, the same rules block explicit and deceptive unsafe commands, preserve a benign object-directed skill, and interrupt an executing strike when a person enters view.
The generator remains frozen during both offline generation and closed-loop execution.
Future work includes extending \csf{} to the root trajectory, letting the robot redirect its locomotion as well as its pose when a rule becomes active, and regenerating safe continuations from the live scene to cover longer, evolving interactions.



\balance
\begingroup
\scriptsize
\bibliographystyle{ieeetr}
\bibliography{references.bib}

@STRING{tog       = "ACM Transactions on Graphics"}

@inproceedings{song2021ddim,
  title={Denoising diffusion implicit models},
  author={Song, Jiaming and Meng, Chenlin and Ermon, Stefano},
  journal={International Conference on Learning Representations (ICLR)},
  year={2021}
}

@inproceedings{lu2022dpm,
  title={Dpm-solver: A fast ode solver for diffusion probabilistic model sampling in around 10 steps},
  author={Lu, Cheng and Zhou, Yuhao and Bao, Fan and Chen, Jianfei and Li, Chongxuan and Zhu, Jun},
  journal={Advances in neural information processing systems},
  volume={35},
  pages={5775--5787},
  year={2022}
}

@inproceedings{lipman2023flow,
  title={Flow matching for generative modeling},
  author={Lipman, Yaron and Chen, Ricky TQ and Ben-Hamu, Heli and Nickel, Maximilian and Le, Matt},
  journal={International Conference on Learning Representations (ICLR)},
  year={2023}
}

@inproceedings{guo2022generating,
  title={Generating diverse and natural 3d human motions from text},
  author={Guo, Chuan and Zou, Shihao and Zuo, Xinxin and Wang, Sen and Ji, Wei and Li, Xingyu and Cheng, Li},
  booktitle={2022 IEEE/CVF Conference on Computer Vision and Pattern Recognition (CVPR)},
  pages={5142--5151},
  year={2022},
  organization={IEEE}
}

@inproceedings{petrovich2022temos,
  title={Temos: Generating diverse human motions from textual descriptions},
  author={Petrovich, Mathis and Black, Michael J and Varol, G{\"u}l},
  booktitle={European conference on computer vision},
  pages={480--497},
  year={2022},
  organization={Springer}
}

@inproceedings{tevet2023mdm,
  title={Human motion diffusion model},
  author={Tevet, Guy and Raab, Sigal and Gordon, Brian and Shafir, Yonatan and Cohen-Or, Daniel and Bermano, Amit H},
  booktitle = {International Conference on Learning Representations (ICLR)},
  year      = {2023}
}

@article{zhang2024motiondiffuse,
  title={Motiondiffuse: Text-driven human motion generation with diffusion model},
  author={Zhang, Mingyuan and Cai, Zhongang and Pan, Liang and Hong, Fangzhou and Guo, Xinying and Yang, Lei and Liu, Ziwei},
  journal={IEEE transactions on pattern analysis and machine intelligence},
  volume={46},
  number={6},
  pages={4115--4128},
  year={2024},
  publisher={IEEE}
}

@inproceedings{chen2023mld,
  title={Executing your commands via motion diffusion in latent space},
  author={Chen, Xin and Jiang, Biao and Liu, Wen and Huang, Zilong and Fu, Bin and Chen, Tao and Yu, Gang},
  booktitle={2023 IEEE/CVF Conference on Computer Vision and Pattern Recognition (CVPR)},
  pages={18000--18010},
  year={2023},
  organization={IEEE}
}

@inproceedings{zhang2023t2mgpt,
  title={Generating human motion from textual descriptions with discrete representations},
  author={Zhang, Jianrong and Zhang, Yangsong and Cun, Xiaodong and Zhang, Yong and Zhao, Hongwei and Lu, Hongtao and Shen, Xi and Shan, Ying},
  booktitle={Proceedings of the IEEE/CVF conference on computer vision and pattern recognition},
  pages={14730--14740},
  year={2023}
}

@inproceedings{cheng2024exbody,
  title={Expressive whole-body control for humanoid robots},
  author={Cheng, Xuxin and Ji, Yandong and Chen, Junming and Yang, Ruihan and Yang, Ge and Wang, Xiaolong},
  booktitle={Robotics: Science and systems},
  year      = {2024}
}

@inproceedings{wang2022humanise,
  title={Humanise: Language-conditioned human motion generation in 3d scenes},
  author={Wang, Zan and Chen, Yixin and Liu, Tengyu and Zhu, Yixin and Liang, Wei and Huang, Siyuan},
  journal={Advances in Neural Information Processing Systems},
  volume={35},
  pages={14959--14971},
  year={2022}
}

@inproceedings{tevet2022motionclip,
  title={Motionclip: Exposing human motion generation to clip space},
  author={Tevet, Guy and Gordon, Brian and Hertz, Amir and Bermano, Amit H and Cohen-Or, Daniel},
  booktitle={European conference on computer vision},
  pages={358--374},
  year={2022},
  organization={Springer}
}

@inproceedings{petrovich2023tmr,
  title={TMR: Text-to-motion retrieval using contrastive 3D human motion synthesis},
  author={Petrovich, Mathis and Black, Michael J and Varol, G{\"u}l},
  booktitle={2023 IEEE/CVF International Conference on Computer Vision (ICCV)},
  pages={9454--9463},
  year={2023},
  organization={IEEE}
}

@inproceedings{karunratanakul2023gmd,
  title={Guided motion diffusion for controllable human motion synthesis},
  author={Karunratanakul, Korrawe and Preechakul, Konpat and Suwajanakorn, Supasorn and Tang, Siyu},
  booktitle={2023 IEEE/CVF International Conference on Computer Vision (ICCV)},
  pages={2151--2162},
  year={2023},
  organization={IEEE}
}

@inproceedings{xie2024omnicontrol,
  title={Omnicontrol: Control any joint at any time for human motion generation},
  author={Xie, Yiming and Jampani, Varun and Zhong, Lei and Sun, Deqing and Jiang, Huaizu},
  booktitle={International Conference on Learning Representations},
  volume={2024},
  pages={28176--28194},
  year={2024}
}

@inproceedings{yuan2023physdiff,
  title={Physdiff: Physics-guided human motion diffusion model},
  author={Yuan, Ye and Song, Jiaming and Iqbal, Umar and Vahdat, Arash and Kautz, Jan},
  booktitle={2023 IEEE/CVF International Conference on Computer Vision (ICCV)},
  pages={15964--15975},
  year={2023},
  organization={IEEE}
}

@inproceedings{shafir2024priormdm,
  title={Human motion diffusion as a generative prior},
  author={Shafir, Yonatan and Tevet, Guy and Kapon, Roy and Bermano, Amit},
  booktitle={International Conference on Learning Representations},
  volume={2024},
  pages={8717--8733},
  year={2024}
}

@inproceedings{schramowski2023sld,
  title={Safe latent diffusion: Mitigating inappropriate degeneration in diffusion models},
  author={Schramowski, Patrick and Brack, Manuel and Deiseroth, Bj{\"o}rn and Kersting, Kristian},
  booktitle={2023 IEEE/CVF Conference on Computer Vision and Pattern Recognition (CVPR)},
  pages={22522--22531},
  year={2023},
  organization={IEEE}
}

@inproceedings{gandikota2023erasing,
  title={Erasing concepts from diffusion models},
  author={Gandikota, Rohit and Materzynska, Joanna and Fiotto-Kaufman, Jaden and Bau, David},
  booktitle={Proceedings of the IEEE/CVF international conference on computer vision},
  pages={2426--2436},
  year={2023}
}

@inproceedings{kumari2023ablating,
  title={Ablating concepts in text-to-image diffusion models},
  author={Kumari, Nupur and Zhang, Bingliang and Wang, Sheng-Yu and Shechtman, Eli and Zhang, Richard and Zhu, Jun-Yan},
  booktitle={2023 IEEE/CVF International Conference on Computer Vision (ICCV)},
  pages={22634--22645},
  year={2023},
  organization={IEEE}
}

@inproceedings{gandikota2024uce,
  title={Unified concept editing in diffusion models},
  author={Gandikota, Rohit and Orgad, Hadas and Belinkov, Yonatan and Materzy{\'n}ska, Joanna and Bau, David},
  booktitle={2024 IEEE/CVF Winter Conference on Applications of Computer Vision (WACV)},
  pages={5099--5108},
  year={2024},
  organization={IEEE}
}

@article{wang2026safemo,
  title={Safemo: Linguistically grounded unlearning for trustworthy text-to-motion generation},
  author={Wang, Yiling and Zhang, Zeyu and Wang, Yiran and Tang, Hao},
  journal={arXiv preprint arXiv:2601.00590},
  year={2026}
}

@article{cho2026safeflow,
  title={SafeFlow: Real-Time Text-Driven Humanoid Whole-Body Control via Physics-Guided Rectified Flow and Selective Safety Gating},
  author={Cho, Hanbyel and Kim, Sang-Hun and Kang, Jeonguk and Koo, Donghan},
  journal={arXiv preprint arXiv:2603.23983},
  year={2026}
}

@article{ames2017cbf,
  title={Control barrier function based quadratic programs for safety critical systems},
  author={Ames, Aaron D and Xu, Xiangru and Grizzle, Jessy W and Tabuada, Paulo},
  journal={IEEE transactions on automatic control},
  volume={62},
  number={8},
  pages={3861--3876},
  year={2016},
  publisher={IEEE}
}

@inproceedings{ames2019control,
  title={Control barrier functions: Theory and applications},
  author={Ames, Aaron D and Coogan, Samuel and Egerstedt, Magnus and Notomista, Gennaro and Sreenath, Koushil and Tabuada, Paulo},
  booktitle={2019 18th European control conference (ECC)},
  pages={3420--3431},
  year={2019},
  organization={Ieee}
}

@article{hsu2024safetyfilter,
  title={The safety filter: A unified view of safety-critical control in autonomous systems},
  author={Hsu, Kai-Chieh and Hu, Haimin and Fisac, Jaime F},
  journal={Annual Review of Control, Robotics, and Autonomous Systems},
  volume={7},
  year={2023},
  publisher={Annual Reviews}
}

@article{yang2025cbfrl,
  title={CBF-RL: Safety filtering reinforcement learning in training with control barrier functions},
  author={Yang, Lizhi and Werner, Blake and de Sa, Massimiliano and Ames, Aaron D},
  journal={arXiv preprint arXiv:2510.14959},
  year={2025}
}

@article{yang2026pacman,
  title={PAC-MAN: Perception-Aware CBF-RL for Whole-Body Safety in Humanoid Dodgeball},
  author={Yang, Lizhi and Li, Junheng and Ames, Aaron D},
  journal={arXiv preprint arXiv:2607.28623},
  year={2026}
}

@inproceedings{alshiekh2018shielding,
  title={Safe reinforcement learning via shielding},
  author={Alshiekh, Mohammed and Bloem, Roderick and Ehlers, R{\"u}diger and K{\"o}nighofer, Bettina and Niekum, Scott and Topcu, Ufuk},
  booktitle={Proceedings of the AAAI conference on artificial intelligence},
  volume={32},
  number={1},
  year={2018}
}

@inproceedings{xiao2025safediffuser,
  title={Safediffuser: Safe planning with diffusion probabilistic models},
  author={Xiao, Wei and Wang, Johnson and Gan, Chuang and Hasani, Ramin and Lechner, Mathias and Rus, Daniela},
  booktitle={International Conference on Learning Representations},
  volume={2025},
  pages={100391--100412},
  year={2025}
}

@inproceedings{yang2026safeflowmatcher,
  title={Safeflowmatcher: Safe and fast planning using flow matching with control barrier functions},
  author={Yang, Jeongyong and Jang, Seunghwan and Han, SooJean},
  booktitle={International Conference on Learning Representations},
  volume={2026},
  pages={54058--54087},
  year={2026}
}

@inproceedings{christopher2024projected,
  title={Constrained synthesis with projected diffusion models},
  author={Christopher, Jacob K and Baek, Stephen and Fioretto, Ferdinando},
  journal={Advances in Neural Information Processing Systems},
  volume={37},
  pages={89307--89333},
  year={2024}
}

@article{li2025hardflow,
  title={HardFlow: Hard-Constrained Sampling for Flow-Matching Models Via Trajectory Optimization},
  author={Li, Zeyang and Alim, Kaveh and Azizan, Navid},
  journal={IEEE Transactions on Pattern Analysis and Machine Intelligence},
  year={2026},
  publisher={IEEE}
}

@article{hildreth1957quadratic,
  title={A quadratic programming procedure},
  author={Hildreth, Clifford},
  journal={Naval research logistics quarterly},
  volume={4},
  number={1},
  pages={79--85},
  year={1957},
  publisher={Wiley Online Library}
}

@inproceedings{ren2023knowno,
title = 	 {Robots That Ask For Help: Uncertainty Alignment for Large Language Model Planners},
  author =       {Ren, Allen Z. and Dixit, Anushri and Bodrova, Alexandra and Singh, Sumeet and Tu, Stephen and Brown, Noah and Xu, Peng and Takayama, Leila and Xia, Fei and Varley, Jake and Xu, Zhenjia and Sadigh, Dorsa and Zeng, Andy and Majumdar, Anirudha},
  booktitle = 	 {Proceedings of The 7th Conference on Robot Learning},
  pages = 	 {661--682},
  year = 	 {2023},
  editor = 	 {Tan, Jie and Toussaint, Marc and Darvish, Kourosh},
  volume = 	 {229},
  series = 	 {Proceedings of Machine Learning Research},
  month = 	 {06--09 Nov},
  publisher =    {PMLR},
}

@inproceedings{yang2024safetychip,
  title={Plug in the safety chip: Enforcing constraints for llm-driven robot agents},
  author={Yang, Ziyi and Raman, Shreyas S and Shah, Ankit and Tellex, Stefanie},
  booktitle={2024 IEEE International Conference on Robotics and Automation (ICRA)},
  pages={14435--14442},
  year={2024},
  organization={IEEE}
}

@inproceedings{robey2025jailbreaking,
  author    = {Alexander Robey and Zachary Ravichandran and Vijay Kumar and Hamed Hassani and George J. Pappas},
  title     = {Jailbreaking {LLM}-Controlled Robots},
  booktitle = {IEEE Int. Conf. Robot. Autom. (ICRA)},
  year      = {2025}
}

@inproceedings{karunratanakul2024dno,
  title={Optimizing diffusion noise can serve as universal motion priors},
  author={Karunratanakul, Korrawe and Preechakul, Konpat and Aksan, Emre and Beeler, Thabo and Suwajanakorn, Supasorn and Tang, Siyu},
  booktitle={2024 IEEE/CVF Conference on Computer Vision and Pattern Recognition (CVPR)},
  pages={1334--1345},
  year={2024},
  organization={IEEE}
}

@inproceedings{liu2024progmogen,
  title={Programmable motion generation for open-set motion control tasks},
  author={Liu, Hanchao and Zhan, Xiaohang and Huang, Shaoli and Mu, Tai-Jiang and Shan, Ying},
  booktitle={2024 IEEE/CVF Conference on Computer Vision and Pattern Recognition (CVPR)},
  pages={1399--1408},
  year={2024},
  organization={IEEE}
}

@article{cao2026texedo,
  title={TEXEDO: Test Time Scaling for Controller-aware Language-conditioned Humanoid Motion Generation},
  author={Cao, Jianuo and Chen, Yuxin and Song, Yuzhen and Tomizuka, Masayoshi and Li, Chenran and Tian, Thomas},
  journal={arXiv preprint arXiv:2606.22998},
  year={2026}
}

@inproceedings{dematteis2026unlearning,
  title={Human motion unlearning},
  author={De Matteis, Edoardo and Migliarini, Matteo and Sampieri, Alessio and Spinelli, Indro and Galasso, Fabio},
  booktitle={Proceedings of the AAAI Conference on Artificial Intelligence},
  volume={40},
  number={5},
  pages={3533--3541},
  year={2026}
}

@inproceedings{miao2025alert,
  title={Autonomous llm-enhanced adversarial attack for text-to-motion},
  author={Miao, Honglei and Ma, Fan and Quan, Ruijie and Zhan, Kun and Yang, Yi},
  booktitle={Proceedings of the AAAI conference on artificial intelligence},
  volume={39},
  number={6},
  pages={6144--6152},
  year={2025}
}

@article{zhao2026ardy,
  title={Autoregressive Diffusion with Hybrid Representation for Interactive Human Motion Generation},
  author={Zhao, Kaifeng and Petrovich, Mathis and Zhang, Haotian and Wang, Tingwu and Tang, Siyu and Rempe, Davis},
  journal={ACM Transactions on Graphics (TOG)},
  volume={45},
  number={4},
  pages={1--14},
  year={2026},
  publisher={ACM New York, NY, USA}
}

@article{rempe2026kimodo,
  title={Kimodo: Scaling controllable human motion generation},
  author={Rempe, Davis and Petrovich, Mathis and Yuan, Ye and Zhang, Haotian and Peng, Xue Bin and Jiang, Yifeng and Wang, Tingwu and Iqbal, Umar and Minor, David and de Ruyter, Michael and others},
  journal={arXiv preprint arXiv:2603.15546},
  year={2026}
}

@article{jia2026echo,
  title={ECHO: Edge-cloud humanoid orchestration for language-to-motion control},
  author={Jia, Haozhe and Song, Jianfei and Zhang, Yuan and Jin, Honglei and Fan, Youcheng and Chen, Wenshuo and Zhang, Wei and Yue, Yutao},
  journal={arXiv preprint arXiv:2603.16188},
  year={2026}
}

@inproceedings{li2026motionhiflow,
    author    = {Li, Heng and Lin, Xiaotong and Zeng, Ling-An and Kang, Yulei and Li, Shuai and Hu, Jian-Fang},
    title     = {MotionHiFlow: Text-to-Motion via Hierarchical Flow Matching},
    booktitle = {Proceedings of the IEEE/CVF Conference on Computer Vision and Pattern Recognition (CVPR)},
    month     = {June},
    year      = {2026},
    pages     = {9352-9363}
}

@book{boyd2004convex,
  author    = {Stephen Boyd and Lieven Vandenberghe},
  title     = {Convex Optimization},
  publisher = {Cambridge University Press},
  year      = {2004}
}

@article{luo2025sonic,
  title={Sonic: Supersizing motion tracking for natural humanoid whole-body control},
  author={Luo, Zhengyi and Yuan, Ye and Wang, Tingwu and Li, Chenran and Casta{\~n}eda, Fernando and Chen, Sirui and Cao, Zi-Ang and Li, Jiefeng and Minor, David and Ben, Qingwei and others},
  journal={Science Robotics},
  volume={11},
  number={117},
  pages={eaed4592},
  year={2026},
  publisher={American Association for the Advancement of Science}
}

@inproceedings{agrawal2017dcbf,
  title={Discrete control barrier functions for safety-critical control of discrete systems with application to bipedal robot navigation.},
  author={Agrawal, Ayush and Sreenath, Koushil},
  booktitle={Robotics: Science and systems},
  volume={13},
  pages={1--10},
  year={2017},
  organization={Cambridge, MA, USA}
}
\endgroup

\end{document}